\documentclass[nohyperref]{article}

\usepackage{microtype}
\usepackage{graphicx}
\usepackage{subfigure}
\usepackage{booktabs} % for professional tables
\usepackage{bbm}

\usepackage{hyperref}

\usepackage[accepted]{icml2023}

\usepackage{amsmath}
\usepackage{amssymb}
\usepackage{mathtools}
\usepackage{amsthm}

\usepackage[capitalize,noabbrev]{cleveref}

\theoremstyle{plain}
\newtheorem{theorem}{Theorem}[section]

\theoremstyle{definition}

\newtheorem{assumption}[theorem]{Assumption}
\theoremstyle{remark}

\newcommand{\argmin}[1]{\underset{#1}{\operatorname{argmin}}\;} 

\usepackage[textsize=tiny]{todonotes}

\icmltitlerunning{~ \hfill Multi-task Highly Adaptive Lasso \hfill \thepage}

\begin{document}

\twocolumn[
\icmltitle{Multi-task Highly Adaptive Lasso}

% It is OKAY to include author information, even for blind
% submissions: the style file will automatically remove it for you
% unless you've provided the [accepted] option to the icml2023
% package.

% List of affiliations: The first argument should be a (short)
% identifier you will use later to specify author affiliations
% Academic affiliations should list Department, University, City, Region, Country
% Industry affiliations should list Company, City, Region, Country

% You can specify symbols, otherwise they are numbered in order.
% Ideally, you should not use this facility. Affiliations will be numbered
% in order of appearance and this is the preferred way.
%\icmlsetsymbol{equal}{*}

\begin{icmlauthorlist}
\icmlauthor{Ivana Malenica}{1}
\icmlauthor{Rachael V. Phillips}{2}
\icmlauthor{Daniel Lazzareschi}{3}
\icmlauthor{Jeremy R. Coyle}{4}
\icmlauthor{Romain Pirracchio}{3}
\icmlauthor{Mark J. van der Laan}{2}
\end{icmlauthorlist}

\icmlaffiliation{1}{Department of Statistics, Harvard University, Cambridge, MA, USA}
\icmlaffiliation{2}{Division of Biostatistics, University of California at Berkeley, Berkeley, CA, USA}
\icmlaffiliation{3}{Department of Anesthesia and Perioperative Care, University of California at San Francisco, San Francisco, CA, USA}
\icmlaffiliation{4}{Preva Group, Seattle, WA, USA}

\icmlcorrespondingauthor{Ivana Malenica}{imalenica@berkeley.edu}

% You may provide any keywords that you
% find helpful for describing your paper; these are used to populate
% the "keywords" metadata in the PDF but will not be shown in the document
\icmlkeywords{multi-task, multivariate, machine learning, highly adaptive lasso, nonparametric statistics, sparsity}
\vskip 0.3in
]

% this must go after the closing bracket ] following \twocolumn[ ...

% This command actually creates the footnote in the first column
% listing the affiliations and the copyright notice.
% The command takes one argument, which is text to display at the start of the footnote.
% The \icmlEqualContribution command is standard text for equal contribution.
% Remove it (just {}) if you do not need this facility.

\printAffiliationsAndNotice{}  % leave blank if no need to mention equal contribution

\begin{abstract}
We propose a novel, fully nonparametric approach for the multi-task learning, the Multi-task Highly Adaptive Lasso (MT-HAL). MT-HAL simultaneously learns features, samples and task associations important for the common model, while imposing a shared sparse structure among similar tasks. Given multiple tasks, our approach automatically finds a sparse sharing structure. The proposed MTL algorithm attains a powerful dimension-free convergence rate of $o_p(n^{-1/4})$ or better. We show that MT-HAL outperforms sparsity-based MTL competitors across a wide range of simulation studies, including settings with nonlinear and linear relationships, varying levels of sparsity and task correlations, and different numbers of covariates and sample size. 
\end{abstract}
 
\section{Introduction}
\label{sec:intro}

Multi-task Learning (MTL) is a machine learning paradigm in which multiple tasks are simultaneously learned by a shared model. Originally motivated by insufficient data issues, MTL proved broadly effective over the years even as $n$ got large ---- attracting a lot of attention in the artificial intelligence and machine learning (ML) communities \cite{ruder2017, crawshaw2020, zhang2022}. Applications of MTL are extensive and cross-disciplinary, ranging from reinforcement learning and speech recognition, to bioinformatics and systems biology \cite{kendall2018, tang2020, dizaji2021, sodhani2021, zhang2022, wang2022}. Deep multi-task learning models have become particularly prevalent in computer vision and natural language processing \cite{misra2016, liu2019, vandenhende2022}. 

The wide application of MTL can be attributed to its connection to how humans learn: a single complicated process can often be divided into several related (sub)tasks. Due to their common latent structure, learning from all (sub)tasks simultaneously proves to be more advantageous than simply learning each one independently. In particular, MTL utilizes shared knowledge to improve generalization performance, as it allows for common representations to be effectively leveraged by all considered tasks. This is in contrast to another ML paradigm with a similar setup, known as transfer learning --- which seeks to improve performance of a single task via knowledge transfer from source tasks \cite{zhuang2019}. Having access to kindred processes, MTL learns more robust representations --- resulting in better knowledge sharing, and lower risk of catastrophic covariate shifts and overfitting. In addition, multi-task learning (1) increases the effective sample size for fitting a model, (2) ignores the task data-dependent noise, (3) biases the model to prefer representations that other tasks also prefer and (4) allows for ``eavesdroping'' across related tasks, as some relationships might be easier to learn in particular tasks \cite{ruder2017}. Shared representations also increase data efficiency, and can potentially yield faster learning speed and better generalizations for future tasks \cite{crawshaw2020}. 

The key challenge in MTL is \textit{what}, \textit{how} and \textit{when} to share the common structure among all (or some subset of) tasks \cite{zhang2022}. In the following, we pay particular attention on \textit{what to share}, which implies \textit{how}. More specifically, the MTL literature on common task representation typically focuses on instance, parameter or feature sharing. Instance-based MTL aims to detect important occurrences which can be useful for other tasks. Parameter-based MTL shares coefficients, weights or layers to help learn model parameters; the most common approaches include low-rank, task clustering, task relation learning and decomposition approaches \cite{zhang2022}. In particular, the task relation approach learns model parameters and pairwise task relations simultaneously \cite{zhang2010,zhang2014}. The feature-based MTL learns common predictors across tasks as a way to share knowledge. As tasks are related, it's intuitive to assume that different tasks share some common feature representation. One particularly interesting way of thought is to learn a subset of original features as the joint task representation; this has prompted a rich literature on sparse structure learning \cite{lounici2009,obozinski2011,sun2019,wang2022}.

In this work, we present the Multi-task Highly Adaptive Lasso (MT-HAL) algorithm, which represents a \textit{fully nonparameteric} method for MTL. Most modern ML algorithms assume similar behavior for all points sufficiently close according to a given metric, referred to as local smoothness assumptions (e.g., adaptive selection of the size of a neighborhood). Reliance on local smoothness is often necessary in order to ensure favorable statistical properties of the estimator, especially in high dimensions. However, too much reliance on it can render an estimator inefficient. Instead of relying on parametric or local smoothness assumptions, the proposed MTL algorithm builds on the Highly Adaptive Lasso (HAL): a nonparametric estimator with remarkable convergence rates that does not rely on local smoothness assumptions and is not constructed using them \citep{benkeser2016,vanderlaan2017hal}. Instead, HAL restricts the global measure of smoothness by assuming that the truth is cadlag (right-hand continuous with left-hand limits) with a bounded sectional variation norm (``HAL'' or $\mathcal{H}$ class). To the best of our knowledge, HAL is the only proven method which converges quickly enough for a large class of functions, independent of the dimension of the covariate space \cite{tsiatis2006,vanderlaan2017hal,schuler2022}.

The proposed algorithm, MT-HAL, is a minimum loss-based estimator in the HAL class of functions with a bounded variation norm. It depends on data-dependent basis functions constructed over all samples, predictors and tasks. The constructed basis functions incorporate task interactions across available samples and predictors. In contrast to the usual feature-based MTL algorithms, MT-HAL does not require a parametric specification of the relationship between predictors and outcome. MT-HAL also does not require any functional form, model or prior knowledge on the relationship between the different tasks. If prior knowledge on task relationships is available, functional forms can be supplied to MT-HAL and the generated basis functions will respect the provided form. Roughly speaking, MT-HAL restricts the amount of true function variation over its domain in a global sense, instead of locally. Practically speaking, it relies on the mixed-norm $l_{2,1}$ to bound the variation norm and produce an unique solution with a similar sparsity pattern across tasks. 

At a high-level, MT-HAL operates as a fully nonparametric mixture of feature learning and task relation approach in MTL. It learns common data-dependent basis functions (aka, features based on original predictors), which are robust and invariant to all available tasks. To avoid including unrelated processes, MT-HAL also includes task interactions as a task membership indicator and across the domain space of each task. As such, it operates as a task relation parameter-based MTL as well --- it learns pairwise task relations simultaneously as model parameters. The final basis function coefficients therefore give us an interpretable insight into a subset of features (and even feature domain), samples, and task relationships relevant for the shared model. The learned common representation reflects a joint sparsity structure as well, due to the final cross-validation based bound on the true variation norm. 

Despite being a powerful fully nonparametric approach to MTL, MT-HAL still preserves its dimension-free $o_p(n^{-1/4})$ rate even in a multi-task setting. As such, MT-HAL is not only suited for finding a common model among tasks, but has convergence rates necessary for semi-parametric efficient estimation as well. We demonstrate superior performance of MT-HAL across many different simulations, testing how (non)linearity, level of sparsity, task relatedness, dimension of covariate space and sample size affect its performance. Finally, we apply MT-HAL to the benchmark Parkinson's disease dataset, which predicts disease clinical symptom scores for each patient based on biomedical voice measurements \cite{tsanas2009accurate, jawanpuria2012}.

%HAL Part
%In contrast to the usual Lasso estimator, the proposed method does not require a parametric specification of the relationship between predictors and outcome, or separate tasks. 

%The MTL community still lacks principles grounded in statistical theory that guide the sharing of information across multiple prediction tasks. Here, we introduce an oracle for MTL and corresponding theoretical results that provide structure for designing rigorous MTL machines.

%While most machine learning algorithms perform well in prediction problems, most lack theoretical convergence rates necessary for non/semi-parametric efficient estimation. 

%When multiple different tasks are related, it can be advantageous to learn all the tasks simultaneously instead of independently. the L2 distance between the parameters of task models is added to the training objective, in order to encourage similar model parameters between different tasks.
 
% insert motivating example based on medical data analysis

\section{Formulation of the Statistical Problem}\label{sec::formulation}
\subsection{Data and Likelihood}\label{sec::data}
Let $O^k$ denote data on the $k^{th}$ task, such that $O^k = (T^k, X^k, Y^k) = (W^k, Y^k)$ for $k=1, \ldots, K$. Each $O^k$ is of a fixed dimension, and an element of a Euclidean set ${\cal O}$. We assume that different tasks for each $k=1, \ldots, K$, and the corresponding data $\{O^k\}_{k=1}^{K}$, are possibly related to each other. Let $X^k$ denote a vector of baseline covariates (``predictors'', ``features'') for the $k^{th}$ task, such that $X^k \in \mathbb{R}^{P_k}$ and $X^k = (X^k_{1}, \ldots, X^k_{P_k})$. We denote $n_k$ as the number of samples for task $k$, which could potentially be different for each task. Then, we have that $X_p^k = \{X_{1,p}^k, \ldots, X_{n_k,p}^k\}$ where $X_{i,p}^k$ represents data on covariate $p$ for sample $i$ in task $k$. Without loss of generality, we assume all predictors in $X^k$ have mean zero and standard deviation one, but otherwise make no assumptions on the dimension, form or correlation structure for the covariate process. As such, for all predictors $p$, we let
\begin{equation*}
    \sum_{i=1}^{n_k} X_{i,p}^k = 0 \ \text{and } \ \sum_{i=1}^{n_k} (X_{i,p}^k)^2 = 1.
\end{equation*}
Further, we define $Y^k$ as a vector of outcomes, where $Y^k = (Y^k_{1}, \ldots, Y^k_{n_k})$, and $Y^k$ is a $n_k \times 1$ vector. We also specifically define $T^k$ as categorical variable indicating $k^{th}$ task membership: if sample $i$ is part of task $k$, then $T^k_i=k$. If we want to emphasize the set of all covariates in task $k$, we write $W^k = (T^k, X^k)$.

In order to define the combined set of all tasks, we specify a more compact notation where $n = \sum_{k=1}^K n_k$. In particular, let ${Y} = (Y^1, \ldots, Y^K)$ and ${T} = (T^1, \ldots, T^K)$ where ${Y}$ and ${T}$ are $n \times 1$ vectors. The matrix of covariates, ${X} = (X^1, \ldots, X^K)$ is a block-diagonal matrix with $X^k$ being the $k^{th}$ block. Let $d = \sum_{k=1}^K P_k$, so that $X$ is a $n \times d$ matrix; if all tasks have the same covariates, then we define $d$ as $d=P$. Consequently, we define $n$ independent and identically distributed (i.i.d.) observations of ${O}_i$ for $i \in [n]$ where ${O}_i = ({T}_i, {X}_i, {Y}_i)$ and ${O}^n = ({O}_1, \ldots, {O}_n) \sim P_0$. With that, we assume that all $K$ tasks are sampled from the same data-generating distribution $P_0$, which is unknown and unspecified. As such, ${O}_i$ reflects observed data for unit $i$ sampled from $P_0$, corresponding to a task specified by $T_i$. 

We denote by $\mathcal{M}$ the statistical model, which represents the set of laws from which ${O}^n$ can be drawn. Intuitively, the more we know (or are willing to assume) about the experiment that produces the data, the smaller the statistical model $\mathcal{M}$. In this work, we define $\mathcal{M}$ as a \textit{nonparametric model}, with $P_0 \in \mathcal{M}$. We impose no assumptions on the true data-generating distribution $P_0$, except that all tasks possibly share some unspecified structure. We further define $P$ as any distribution which also lies in $\mathcal{M}$, such that $P \in \mathcal{M}$. Let $p_0$ denote the density of $P_0$ with respect to (w.r.t) a measure $\mu$ that dominates all elements of $\mathcal{M}$. The likelihood of ${O}^n$, $p_0({O}^n)$, can be factorized according to the time-ordering as follows:
\begin{align}
\label{eq:likelihood}
%p_0(O^n) &= 
\prod_{i=1}^{n}  p_{0,(t,x)}({T}_i,{X}_i) p_{0,y}({Y}_i \mid {T}_i,{X}_i),
\end{align}
where $p_{0,(t,x)}$ marks the probability density for all the baseline covariates, and $p_{0,y}$ is the conditional density of the outcome given all the predictors. At times, it proves useful to use notation from empirical process theory. Specifically, we define $Pf$ to be the empirical average of the function $f$ w.r.t.~the distribution $P$, that is, $Pf = \int f(o)dP(o)$. We use $P_n$ to denote the empirical
distribution of the sample, which gives each observation weight $1/n$, irrespective of the task membership.

\subsection{Target Parameter}\label{sec::target}

We define the relevant feature of the true data distribution we are interested in as the \textit{statistical target parameter} (short, target parameter or estimand). We define a parameter mapping $\Psi : \mathcal{M} \rightarrow \mathbb{R}$, and a parameter value $\psi := \Psi(P)$ for any given $P \in \mathcal{M}$. The estimate, evaluated at $(T,X)$, is written as $\psi(T,X) := \Psi(P)(T,X)$. In some cases, we might be interested in learning the entire conditional distribution $P_{0,y}$. However, frequently the actual goal is to learn a particular feature of the true distribution that satisfies a scientific question of interest. In a multi-task problem, we are interested in multivariate prediction of the entire set of outcomes collected, given the task and observed covariates. The target parameter then corresponds to
\begin{equation}\label{eq::psi}
    \Psi(P_0)(T,X) = \mathbb{E}_{P_0}(Y \mid T,X) =  \mathbb{E}_{P_0}(Y \mid W),
\end{equation}
where the expectation on the right hand side is taken w.r.t the truth, and the true parameter value is denoted as $\psi_0 = \Psi(P_0)$. In words, we are interested in learning $(T,X) \mapsto \Psi(P_0)(T,X)$ using all the data and available tasks. The prediction function for unit $i$ is then obtained with $\Psi(P)(T_i,X_i)$, for any $P \in \mathcal{M}$ and $i \in [n]$. 

\subsection{Loss-based Parameter Definition}\label{sec::loss}

We define $L$ as a valid loss function, chosen in accordance with the target parameter. Specifically, we refer to a valid loss for a given target as a function whose expectation w.r.t. $P_0$ is minimized by the true value of the parameter. Our accent on appropriate loss functions strives from their multiple use within our framework --- as a theoretical criterion for comparing the estimator and the truth, as well as a way to compare multiple estimators of the target parameter \citep{dudoit2003b, dudoit2003a, vaart2006, laan2006oracle}.

Let $L$ be a loss adapted to the problem in question, i.e. a function that maps every $\Psi(P)$ to $L(\Psi(P)) : (k, X_i, Y_i) \mapsto L(\Psi(P))(k, X_i, Y_i)$; note that we can equivalently write $L(\Psi(P))(k, X_i, Y_i)$ as $L(\Psi(P))(T_i, X_i, Y_i)$ or just $L(\Psi(P))(W_i, Y_i)$ for short notation. We define the true risk as the expected value of $L(\Psi(P))(T_i, X_i, Y_i)$ w.r.t the true conditional distribution $P_{0}$ across all individuals and tasks:
\begin{equation*}
    R(P_0, \psi) = \mathbb{E}_{P_0}[L(\Psi(P))(T, X, Y)].
\end{equation*}
The notation for the true risk, $R(P_{0}, \psi)$, emphasizes that $\psi$ is evaluated w.r.t. the true data-generating distribution ($P_{0}$). As specified by the definition of a valid loss, we define $\psi_0$ as the minimizer over the true risk of all evaluated $\psi$ in the parameter space $\pmb{\Psi}$,
\begin{equation*}
\psi_0 = \argmin{\psi \in \pmb{\Psi}} R(P_{0},\psi).
\end{equation*}

The estimator mapping, $\hat{\Psi}$, is a function from the empirical distribution to the parameter space. In particular, let $P_{n,K}$ denote the empirical distribution of $n$ samples collected across $K$ tasks. Then, $P_{n,K} \mapsto \hat{\Psi}(P_{n,K})$ represents a mapping from $P_{n,K}$ into a predictive function $\hat{\Psi}(P_{n,K})$. Further, the predictive function $\hat{\Psi}(P_{n,K})$ maps $(T_i, X_i)$ into a subject-specific outcome, $Y_{i}$. We write $\psi_{n}(T_i, X_i) := \hat{\Psi}(P_{n,K})(T_i, X_i)$ as the predicted outcome for unit $i$ of the estimator $\hat{\Psi}(P_{n,K})$ based on $(T_i, X_i)$. Finally, we note that the true risk establishes a true measure of performance of the estimator. In order to obtain an unbiased estimate of the true risk, we resort to appropriate cross-validation (CV).

\subsection{Cross-validation}\label{sec::cv}

Let $C(i,k)$ denote, at a minimum, the task $k$- and unit $i$-specific record $C(i,k,\cdot) = (O_i,k,\cdot)$. To derive a general representation for cross-validation, we define a split vector $B_n$, where $B_n(i, k) \in \{0,1\}^n$. A realization of $B_n$ defines a particular split of the learning set into corresponding disjoint subsets,
\[
  B_n^v(i,k)=\begin{cases}
                0, \ \ C(i,k) \  \text{in the training set}\\
                1, \ \ C(i,k) \  \text{in the validation set,}
            \end{cases}
\]
where $B_n^v(i,k)$ denotes a $v$-fold assignment of unit $i$ and task $k$ for split $B_n^v$. For example, we can partition the full data into $V$ splits of approximately equal size, such that each fold $v$ for $v = 1, \ldots, V$ has a balanced number of tasks members. We define $P_{n,K,v}^0$ and $P_{n,K,v}^1$ as the empirical distribution of the training and validation sets corresponding to sample split $v$, respectively. To alleviate notation, we let all sample indexes and tasks used for training be an element of $\mathcal{B}_{v}^0$, and $\mathcal{B}_{v}^1$ if in a validation set for fold $v$. The total number of samples in the training and validation set are then written as $n^0$ and $n^1$.

\section{Highly Adaptive Lasso}\label{sec::hal}
We define the multi-task setup and its objective as a statistical parameter estimated via a loss-based paradigm in a nonparametric model, which involves minimizing the cross-validated risk. In the following, we elaborate on a specific algorithm with favorable statistical properties we build on in order to accommodate the multi-task objective.

Most modern nonparametric machine learning algorithms, including neural networks, tree-based models and histogram regression, assume similar behavior for all points sufficiently close according to a given metric \citep{benkeser2016}. We refer to such assumption as \textit{local smoothness}, which typically translates to implicit or adaptive selection of the size of a neighborhood and continuous differentiability. Reliance on local smoothness is often necessary in order to ensure favorable statistical properties of the estimator, especially in high dimensions. However, too much reliance on local smoothness can render an estimator inefficient. While some methods provide a way to calculate neighborhood size that optimizes the bias-variance trade-off (e.g., kernel regression), it is generally unknown how to adjust local smoothness assumptions when constructing a nonparameteric estimator.  

In contrast to most commonly used machine learning methods, the Highly Adaptive Lasso (HAL) is a nonparametric estimator that does not rely, and is not constructed, using local smoothness assumptions \citep{benkeser2016,vanderlaan2017hal}. Instead, HAL restricts a \textit{global measure of smoothness} by assuming that the true target parameter is cadlag (right-hand continuous with left-hand limits) with a bounded sectional variation norm (``HAL" or $\mathcal{H}$ class). In practice, such assumptions prove to be very mild as (1) cadlag functions are very general, even allowing for discontinuities; (2) the variation norm can be made arbitrarily large and adapted to the problem \citep{schuler2022}. To the best of our knowledge, HAL is the only algorithm with fast-enough convergence rates to allow for efficient inference in nonparametric statistical models regardless of the dimension of the problem. In particular, its minimum rate does not depend on the underlying local smoothness of the true regression function, $\psi_0$. In the following, we give a brief theoretical and practical description of HAL, necessary in order to understand the proposed algorithm. The Highly Adaptive Lasso for the multi-task objective is presented in Section \ref{sec::multi_task_hal}.

\subsection{Cadlag functions with finite variation norm}

We assume that $\psi_0 \in \mathcal{H}$, where $\mathcal{H}$ is a set of $d$-variate real valued cadlag functions with $\tau$ as an upper bound on all support. A function in $\mathcal{H}$ is right-continuous with left-hand limits, as well as left-continuous at any point on the right-edge of $[0,\tau]$. In addition, we also assume that $\psi_0$ has a finite sectional variation norm bounded by some universal constant $M < \infty$. Note that any cadlag function with bounded variation norm generates a finite measure, resulting in well defined integrals w.r.t the function. With that, we define a sectional variation norm of a multivariate real valued cadlag function $\psi$ as
\begin{equation*}
    \|\psi\|_{var} = \psi(0) + \sum_{s \subset \{1, \ldots, d\}} \int_{0_s}^{\tau_s} |\psi_s(du_s)|,
\end{equation*}
where the sum is taken over all subsets of $\{1, \ldots, d\}$. In particular, for a given subset of $s$, we define $u_s = (u_j : j \in s)$ and $u_{-s} = (u_j : j \notin s)$. To illustrate for $X=(X_1,X_2)$ and $s=1$, we have that $X_s = \{X_1\}$ and $X_{-s} = \{X_2\}$. We define the section $\psi_s(u_s) \equiv \psi(u_s, 0_{-s})$, where $\psi_s$ varies along the variables in $u_s$ according to $\psi$, but sets the variables in $u_{-s}$ to zero. 

\subsection{Minimum Loss-based Estimator in the HAL class}

We denote a class of functions $\mathcal{H}_M$ as $\mathcal{H}_M = \{\psi : \|\psi\|_{var} < M\}$, where $M$ is an arbitrary large and unknown constant. As introduced in subsection \ref{sec::loss}, we define the true minimizer of the average loss in class $\mathcal{H}_M$ as 
\begin{equation*}
\psi_{0,M} = \argmin{\psi \in \mathcal{H}_M} R(P_{0},\psi), 
\end{equation*}
with the MLE in $\mathcal{H}_M$ defined as 
\begin{equation*}
\psi_{n} = \psi_{n,M} = \argmin{\psi \in \mathcal{H}_M} \frac{1}{n} \sum_{i=1}^n L(\psi)(T_i,X_i,Y_i).
\end{equation*}
We emphasize that if $\|\psi_0\|_{var} < M$, then $\psi_{0,M} = \psi_0$ as $\psi_0 \in \mathcal{H}_M$. As most functions with infinite variation norm tend to be pathological (e.g., $\sin(1/x)$ or $x \sin(1/x)$), having $\|\psi_0\|_{var} < M$ is a very mild assumption. 

The true bound on the sectional variation norm is, however, unlikely to be known in practice. One way of choosing $M$ is based on the cross-validated choice of the bound. To start, we consider a grid of potential values, where $M_1$ is the smallest and $M_{B}$ the largest bound such that $\|\psi_0\|_{var} < M_B$. For each $b = 1, \ldots, B$ corresponding to bounds $M_1, \ldots, M_B$, we have the parameter value $\psi_{n,M_b} = \hat{\Psi}_{M_b}(P_{n,K})$, where $\psi_{n,M_b}$ is the MLE in the $\mathcal{H}$ class with variation norm smaller than $M_b$ ($\mathcal{H}_{M_b}$). Referring to the cross-validation notation introduced in subsection \ref{sec::cv}, the cross-validation selector $M_n$ of $M$ is the grid value with the lowest estimated cross-validated risk:
\begin{equation*}
    M_n = \argmin{b} \frac{1}{V} \sum_{v=1}^V \mathbb{E}_{P_{n,K,v}^1} L(\hat{\Psi}_{M_b}(P_{n,K,v}^0)).
\end{equation*}
Finally, in order to study the difference between an estimator and the truth (relevant for theoretical results presented in subsection \ref{sec::theory}), we construct loss-based dissimilarity measures. In particular, for some $\psi \in \mathcal{H}$, we define $d_0(\psi, \psi_0)$ as the loss-based dissimilarity measure corresponding to the squared error loss where
\begin{equation}\label{eqn::diss}
    d_0(\psi, \psi_0) = \mathbb{E}_{P_0}[L(\psi) - L(\psi_0)] = \|\psi - \psi_0\|_{P_0}^2.
\end{equation}
 
\section{Multi-task Highly Adaptive Lasso}\label{sec::multi_task_hal}

Keeping in mind the underpinnings presented in Section \ref{sec::hal}, how can we construct a HAL estimator for the multi-task objective? 
%that has the same rate of convergence as stated in Theorem \ref{thm::hal_convergence}? 
Practically, the original HAL is a minimum loss-based estimator in the $\mathcal{H}$ class. 
%which can be computed using $l_1$-penalized regression \citep{vanderlaan2017hal,benkeser2016}. 
Instead of specifying a relationship between outcome and covariates, HAL uses a special set of data-dependent basis functions. As such, all the estimation is done completely nonparametrically, only restricting the global amount of fluctuation of the true function over its domain (instead of locally). 

For the multi-task problem however, we want all tasks to be learned simultaneously, instead of independently. We also want it to be agnostic to the number of predictors and samples per task, as well as to learn task-specific associations at the same time as the shared model. In case parameters for different tasks  share the same sparsity pattern, we want the proposed algorithm to be able to learn that structure too. Hence, when estimating the target parameter in Equation \eqref{eq::psi}, we opt for a joint sparsity regularization, instead of penalizing each task separately. Similarly to the multivariate group lasso \citep{yuan2006,obozinski2011,simon2012}, we consider a $l_1/l_q$ mixed-norm penalty with $q>1$, which for $l_1/l_2$ corresponds to 
\begin{equation*}
    \|\alpha\|_{2,1} := \left[\sum_{p=1}^d \left(\sum_{k=1}^K |\alpha_{p}^k|^{2} \right)^{1/2}  \right] = \sum_{p=1}^d \|\alpha_p\|_2,
\end{equation*}
for some vector of coefficients $\alpha$, $K$ tasks, and $d$ predictors. The $\alpha_p^k$ then represents coefficients for predictor $p$ and task $k$.
%, of dimension $n_k \times 1$.
Intuitively, norm $\|\alpha_p\|_2$ has the effect of enforcing the elements of $\alpha_p$ to achieve zeros simultaneously. The objective of the $l_{2,1}$ norm is to minimize the $l_2$ norm of the columns, followed by the $l_1$ norm over all predictors --- resulting in a matrix with a sparse number of columns and a small $l_2$ norm. The mixed-norm penalty allows us to impose a soft constraint on the structure of the sparse solution that we are looking for, thus promoting a group-wise sparsity pattern across multiple tasks.

\subsection{Algorithm}\label{sec::algorithm}
In order to present the full algorithm, we first note that a function with finite variation norm can be represented as the sum over subsets of an integral w.r.t. a subset-specific measure. Also, each subset-specific measure can be approximated by a discrete measure over support points; we elaborate on this derivation in the Appendix. As a consequence, we can define support by the actual $n$  observations across $K$ tasks, which reduces the optimization problem to a finite-dimensional one. Let  $\Phi : \mathcal{W} \rightarrow \{0,1\}^{N}$ define a mapping where neither $\Phi$ or $N$ are pre-specified, but instead completely determined by the data (and thus data-adaptive). With that, let $\phi_{s,t}$ denote a particular basis function, where $s$ is a section in $\{1,\ldots,d\}$, and $t$ a knot point corresponding to observed $W_i$. To enumerate all bases, we have that $\Phi(W) = [\phi_{s_1,t_1}(W),\ldots, \phi_{s_N,t_N}(W)]$, resulting in a total of $N$ basis functions, where $N=n(2^d-1)$. We emphasize that the bases depend only on \textit{observed values}, across all samples and tasks. In particular, let $\Tilde{w}_{s,i}$ denote an observed value of $W$, where $\Tilde{w}_{s,i} = \{\Tilde{w}_{c,i} : c \in s\}$ for subset $s$ with $i = 1, \ldots, n_k$ in task $k$. Then, we define $\phi_{s,i} = \mathbbm{1}(\Tilde{w}_{s,i} \leq w_s)$, where $\Phi$ constructs all the bases that can be constructed using each observed $W_i$ as a knot point, across all possible sections. As such, applying $\Phi$ row-wise to $W$ results in an output $\Phi(W)$ of zeros and ones (as per indicator function dependent on observed values across all samples and tasks) , which is of $\sum_{k=1}^k n_k \times N$ dimension:
\[
\Phi (
\underbrace{\begin{bmatrix}
W^1 & &  \\
& \ddots & \\
 & & W^k
\end{bmatrix}}_{n \times d}
) \rightarrow 
\underbrace{\begin{bmatrix}
1 & \ldots & 0 & 0 & \ldots & 0 \\
& & \ddots & & \\
0 & \ldots & 0 & 1 & \ldots & 1
\end{bmatrix}}_{n \times N}
\]

Note that passing all the tasks to $\Phi$ at once allows us to \textit{share basis functions across tasks}. What's more, it allows us to create task-interaction bases with separate coefficients. We consider a minimization problem over all linear combinations of the basis functions $w \rightarrow \phi_{s,i}(w) = \mathbbm{1}(\Tilde{w}_{s,i} \leq w_s)$ and corresponding coefficients $\beta_{s,i}$, summed over all subsets $s$ and for each $i$ (as such, over all tasks). Here we emphasize that the sum of the absolute values of the coefficients is the variation norm, and must therefore be bounded. A discrete approximation 
of $\psi$ with support defined by the observed $n$ points may then be constructed as
\begin{align*}
    \psi_{\beta} &= \beta_0 + \sum_{s \subset \{1, \ldots, d\}} \sum_{k=1}^K 
    \sum_{i=1}^{n_k} \beta_{s,i} \mathbbm{1}(\Tilde{w}_{s,i} \leq w_s) \\
    &= \beta_0 + \sum_{s \subset \{1, \ldots, d\}} \sum_{k=1}^K 
    \sum_{i=1}^{n_k} \beta_{s,i}\phi_{s,i},
\end{align*}
where $\beta_0$ is the intercept corresponding to $\psi(0)$. The empirical minimization problem over all discrete measures with support defined by the observed samples then corresponds to
\begin{equation*}
    %\psi_{\beta} = \argmin{\psi_{\beta} \in \mathcal{H}_M} P_n L(\psi_{\beta})(W,Y) + \lambda\|\beta\|_{2,1}, 
    \beta_n = \argmin{\beta, \|\beta\|_{2,1} < M} P_n L(\psi_{\beta}), 
\end{equation*}
where we remind that $P_n$ denotes the empirical distribution of the sample and $\beta$ is a vector of coefficients for each basis with a subspace
\[
\mathcal{H}_{n,M} = \begin{cases}
    \Phi(W)\beta \\
    \text{s.t.} \ \|\beta\|_{2,1} \leq M.
    \end{cases}
\]
As our parameter of interest is a conditional mean, we could use the weighted square error to define the loss. Accordingly, for the $i^{\text{th}}$ sample we have that
\begin{align*}
    L(\psi_{\beta})(W_i,Y_i) 
    &= g_i^k (Y_i^k - \psi_{\beta}(W_i^k))^2,
\end{align*} 
with $g_i^k$ as the sample- and task- specific weight. The multi-task HAL estimator is then formulated by minimizing the squared loss with the $l_{2,1}$ penalty:
\begin{align*}
    \psi_{\beta} = \argmin{\psi_{\beta} \in \mathcal{H}_M} \sum_{k=1}^K \sum_{i=1}^{n_k} L(\psi_{\beta})(W_i,Y_i) 
    %&w_i^k (Y_i^k - \psi_{\beta}(X_i^k, T_i^k))^2 \\
    + \lambda \sum_{p=1}^N \|\beta_p\|_2, 
\end{align*}
where $$\|\beta_p\|_2 = \sqrt{\sum_{k=1}^K |\beta_{p}^k| ^2} =  \sqrt{\sum_{k=1}^K \sum_{i=1}^{n_k} |\beta_{i,p}| ^2}.$$

\begin{algorithm}[tb]
   \caption{Multi-task Highly Adaptive Lasso}
   \label{alg:mtHAL}
\begin{algorithmic}[1]
   \STATE {\bfseries Input:} data $O^k=(W^k,Y^k)$ for tasks $k=1, \ldots, K$, loss function $L$, number of CV folds $V$ \\
   {\bfseries Algorithm:}
   %\REPEAT
   \STATE Create a combined set of tasks $O=(W,Y)$.
   \STATE Define a data-adaptive mapping $\Phi : \mathcal{W} \rightarrow \{0,1\}^{N}$.
   \STATE Generate basis functions across all tasks, predictors and samples:
   $\Phi(W) = [\phi_{s_1,1},\ldots, \phi_{s_N,n}].$
   \STATE Define a grid $M_1, \ldots, M_B$ where $\|\psi_0\|_{\text{var}} < M_B$.
   \FOR{$v=1$ {\bfseries to} $V$}
   \FOR{$b=1$ {\bfseries to} $B$}
   \STATE Empirical minimization for the training set:
   $$\beta_{n.v,b} = \argmin{\beta, \|\beta\|_{2,1} < M_b} P_{n,K,v}^0 L(\psi_{\beta})$$
   \ENDFOR
   \ENDFOR
   \STATE Pick the CV selector in the validation set:
   \begin{equation*}
    M_n = \argmin{b \in \{1, \ldots, B\}} \frac{1}{V} \sum_{v=1}^V P_{n,K,v}^1 L(\psi_{\beta_{n.v,b}}).
\end{equation*}
\end{algorithmic}
\end{algorithm}

\subsection{Theoretical Properties}\label{sec::theory}

For the conditional mean as a target parameter, the MLE ($\psi_{n}$) converges to its $M$-specific truth in loss-based dissimilarity at a rate faster than $n^{1/2}$ (equivalently, no slower than $n^{1/4}$ w.r.t to a norm $l_2$ under $P_0$). This result holds even if the parameter space is completely nonparametric, and, remarkably, regardless of the dimension of $W$. As such, the minimum rate of convergence does not depend on the underlying smoothness of $\psi_0$, as is typically the case for machine learning algorithms. The result still applies even if $\psi_0$ is non-differentiable \citep{benkeser2016, vanderlaan2017hal}. Under weak continuity conditions, the HAL based estimator is also uniformly consistent \citep{laan2017uniform_consist_hal}. Theorem \ref{thm::hal_convergence} formalizes the theoretical properties of HAL with a multi-task objective; in the Appendix, we show that the proposed formulation adapted to the multi-task problem is another valid way of bounding the sectional variation norm. As such, the multi-task HAL is a fast converging algorithm regardless of the dimension of the covariate space or the considered number of tasks $K$. 

\begin{assumption}\label{ass::A1}
\begin{equation*}
    \sup_{\psi \in \mathcal{H}_M}\frac{\|L(\psi_{})\|_{var}}{\|\psi_{}\|_{var}} < \infty.
\end{equation*}
\end{assumption}

\begin{assumption}\label{ass::A2}
\begin{equation*}
    \sup_{\psi_{} \in \mathcal{H}_M}\frac{\|L(\psi_{}) - L(\psi_{0,M})\|^2_{P_0}}{\|\psi_{0} - \psi_{0,M}\|^2_{P_0}} < \infty.
    %\frac{\mathbb{E}_{P_0}[L(\psi_{n,M}) - L(\psi_{0,M})]^2}{\mathbb{E}_{P_0}[L(\psi_0) - L(\psi_{0,M})]} < \infty
\end{equation*}
\end{assumption}

\begin{assumption}\label{ass::A3}
\begin{equation*}
    \sup_{\psi_{} \in \mathcal{H}, o \in \mathcal{O}} |L(\psi)(o)| < \infty.
\end{equation*}
\end{assumption}

\begin{theorem}[Rate of convergence]\label{thm::hal_convergence}
For a square error loss and under Assumptions \eqref{ass::A1} and \eqref{ass::A2}, we have that
\begin{equation*}
    \|\psi_{n,M} - \psi_{0,M}\|_{P_0} = O_P(n^{-1/4-1/(8(d+1))}).
\end{equation*}
If $M_n$ is picked based on minimizing the CV risk, under additional Assumption \eqref{ass::A3}, we then have that 
\begin{align*}
    \|\psi_{n,M_n} - \psi_{0,M}\|_{P_0} = &O_P(n^{-1/4-1/(8(d+1))}) \\
    &+ O_P(n^{-1/2 \log{B}}), 
\end{align*}
for $b=1,\ldots,B$ where $M_{B}$ corresponds to the largest bound on the variation norm in a grid with $\| \psi_0 \|_{\text{var}} < M_{B}$.
\end{theorem}

\begin{proof}
Let $l(\psi_n,\psi_0) = L(\psi_{n,M}) - L(\psi_{0,M})$, where $l(\psi_n,\psi_0)$ falls in the $P_0$-Donsker class with a variation norm smaller than $M$. Then, 
\begin{align*}
    0 &\leq d_0(\psi_{n,M}, \psi_{0,M}) = \mathbb{E}_{P_0}[l(\psi_n,\psi_0)] \\
    &= - \mathbb{E}_{P_n}[l(\psi_n,\psi_0)] + \mathbb{E}_{P_0}[l(\psi_n,\psi_0)]
    + \mathbb{E}_{P_n}[l(\psi_n,\psi_0)] \\
    &= -(P_n - P_0)[l(\psi_n,\psi_0)] + P_n[l(\psi_n,\psi_0)] \\
    &\leq -(P_n - P_0)[l(\psi_n,\psi_0)], 
\end{align*}
where first inequality follows from the definition of a loss-based dissimilarity defined in Equation \ref{eqn::diss}. By Assumptions \eqref{ass::A1} and \eqref{ass::A2}, we have that $P_0[l(\psi_n,\psi_0)]$ and $\|l(\psi_n,\psi_0)\|_{P_0}^2$ are $O_p(n^{-1/2})$. By results in empirical process theory, we know that $\sqrt{n}(P_n - P_0)l(\psi_n,\psi_0) \rightarrow_p 0$ and $P_0 l(\psi_n,\psi_0)^2 \rightarrow_p 0$ as $n \rightarrow \infty$ since $l(\psi_n,\psi_0)$ is Donsker \cite{vanderVaartWellner96}. Therefore, we have that $(P_n - P_0)[l(\psi_n,\psi_0)] = o_P(n^{-1/2})$. As $0 \leq d_0(\psi_{n,M}, \psi_{0,M}) \leq -(P_n - P_0)[l(\psi_n,\psi_0)]$, it follows that $\|\psi_{n,M} - \psi_{0,M} \|_{P_0}^2 = o_P(n^{-1/2})$ and $\|\psi_{n,M} - \psi_{0,M} \|_{P_0} = o_P(n^{-1/4})$. If $M$ is not known, it remains to compare the performance of $\psi_{n,M_n}$ (where $M_n$ is the choice of $M$ which minimizes the CV risk) to the oracle $M_n^*$ (where $M_n^*$ is the choice of $M$ which minimizes the true CV risk). The $O_P(n^{-1/2 \log{B}})$ comes from oracle results for the CV selector w.r.t. a loss-based dissimilarity \cite{dudoit2003a,dudoit2003b,vaart2006,laan2006oracle}. For the precise rate, we refer to van der Laan \yrcite{vanderlaan2017hal}. 
\end{proof}

The result of Theorem \ref{thm::hal_convergence} shows that, roughly, $d_o(\psi_{n,M},\psi_{0,M}) = o_P(n^{-1/2})$ or $\|\psi_{n,M} - \psi_{0,M}\|_{P_0} = o_P(n^{-1/4})$. What's more, Theorem \ref{thm::hal_convergence} states that by choosing $B$ such that $n^{-1/2 \log{B}} \rightarrow 0$ as $n \rightarrow \infty$, we preserve the rate even when the true variation norm is unknown. Note that if $\|\psi_0\|_{var} < M$, it follows that $\|\psi_{n,M} - \psi_{0}\|_{P_0} = O_P(n^{-1/4-1/(8(d+1))})$. 
%Bibaut and van der Laan \yrcite{bibaut2019} improve on the rate in Theorem \ref{thm::hal_convergence} by utilizing the entropy integral of $\mathcal{H}_M$, with the final reported rate of $O_P(n^{-1/3}\log(n)^{2(d-1)/3})$. 

\section{Simulations}\label{sec::simulations}

In this section, we report performance of the proposed multi-task HAL in various simulation settings. In total, we consider four different data-generating distributions and test popular multi-task algorithms based on sparsity assumptions in addition to MT-HAL (MT-lasso and MT-L21). Simulation setups are indicated by the first letter of each setting: (1) ``N" for nonlinear data-generating process; (2) ``H" for high level of sparsity ($60\%$) vs. ``L" for low sparsity ($20\%$); (3) ``S" for the same level of sparsity across tasks vs. ``D" for different sparsity levels.  As such, simulation setup ``NHS'' stands for nonlinear DGP with high level of sparsity that is shared across all tasks. In all simulations we consider $K=5$ tasks and different number of samples for each task ($n=\{100,100,150,150,100\}$ corresponding task 1 to 5). Although the number of covariates was the same across all tasks ($d=6$), we emphasize that our setup allows for different (and non-overlapping) set of covariates. For each multi-task method and simulation setup, we report the mean squared error (MSE) and variable selection performance (precision and accuracy), calculated on a separate test data over $100$ Monte Carlo simulations. Note that, in this setting, true positives denote the true estimated nonzero coefficients. Therefore, precision is calculated as true positives/(true positives + false positives), whereas accuracy is reported as the number of true positives/(true positives + false negatives). 

The data-generating processes corresponding to simulation ``NHS'', ``NLS'', ``NHD'' and ``NLD'' are nonlinear models with $d=6$ covariates and normally distributed coefficients. In particular, true coefficients for simulation ``NHS'' are $\beta_{0,NHS} = (\beta_{0,1}, \ldots, \beta_{0,|NHS|}, 0, \ldots, 0)$, where each $\beta_{0,j}$ for $1 \leq j \leq |NHS|$ is sampled from a standard normal distribution. Similarly, the error term is normally distributed (sampled from $\mathcal{N}(0,0.1)$), with coefficient $0.3$. Covariates with nonzero coefficients, constituting of binary, categorical and continuous variables, are transformed via logarithmic, cosine and squared operations of predictor interactions. For simulations with the same level of sparsity across tasks, highest level of sparsity was $60\%$ (corresponding to around $3$ nonzero coefficients per task), whereas the lowest was $20\%$ ($5$ nonzero coefficients per task). For different sparsity profiles across tasks, we considered random deviations from true level of sparsity across different tasks: $40\%-80\%$ and $0.05\%-40\%$ for high and low level of sparsity, respectively. 

Compared to other sparsity-based multi-task algorithms, MT-HAL results in the lowest MSE across all considered data-generating processes, often reducing the MSE by more than a half. On average, it also tends to produce less false negatives than considered competitors, resulting in high coefficient accuracy and comparable precision (slightly more false positives on average in our simulations: this is to be expected due to a rich space considered, and could be mediated with a finer grid). All algorithms perform best in low sparsity settings, particularly if the sparsity structure is shared across tasks. We report results for nonlinear simulations at $n=600$ in Table \eqref{table::table_sim1}. Additional simulation results are available in the Appendix, and include performance of the proposed method in the following settings: (1) different sample sizes, including very small $n$; (2) linear data-generating processes with interactions and (3) high-dimensional nonlinear setup.

% Note use of \abovespace and \belowspace to get reasonable spacing
% above and below tabular lines.

\begin{table}[ht!]
\vspace{-1em}
\caption{Mean squared error (MSE), precision (``Prec") and accurarcy (``Accu") for each of the nonlinear simulation setups with $d=6$ covariates, $K=5$ tasks, and total of $n=600$ samples split between the $K$ tasks as $\{100,100,150,150,100\}$.}\label{table::table_sim1}
\vskip 0.15in
\begin{center}
\begin{small}
\begin{sc}
\begin{tabular}{@{} @{} lrrrrr @{} @{}}
\toprule
\textbf{Setup} & \textbf{Method} & \textbf{MSE} & \textbf{Prec $\%$}  & \textbf{Accu $\%$} \\
\midrule
\centering
% NHS & MT-HAL   & 0.675 & 58.7 & 75.0  \\
NHS & MT-HAL   & 0.68 & 58.7 & 75.0  \\
NHS & MT-lasso & 1.49 & 64.4 & 31.0  \\
NHS & MT-L21   & 1.48 & 55.3 & 40.9  \\
\midrule
% NLS & MT-HAL   & 0.679 & 91.4 & 80.5  \\
NLS & MT-HAL   & 0.68 & 91.4 & 80.5  \\
NLS & MT-lasso & 1.59 & 98.5 & 44.6  \\
NLS & MT-L21   & 1.54 & 97.6 & 60.5  \\
\midrule
% NHD & MT-HAL   & 0.454 & 61.9 & 69.0  \\
% NHD & MT-lasso & 0.735 & 65.5 & 28.0  \\
% NHD & MT-L21   & 0.714 & 53.5 & 36.7  \\
NHD & MT-HAL   & 0.45 & 61.9 & 69.0  \\
NHD & MT-lasso & 0.74 & 65.5 & 28.0  \\
NHD & MT-L21   & 0.71 & 53.5 & 36.7  \\
\midrule
% NLD & MT-HAL   & 0.666 & 87.9 & 80.5  \\
NLD & MT-HAL   & 0.67 & 87.9 & 80.5  \\
NLD & MT-lasso & 1.55 & 98.0 & 45.2  \\
NLD & MT-L21   & 1.52 & 89.5 & 59.9  \\
\bottomrule
\end{tabular}
\end{sc}
\end{small}
\end{center}
\vskip -0.2in
\end{table}

\section{Data Analysis}\label{sec::data_analysis}

Using a range of biomedical voice measurements obtained from 42 adults with early-stage Parkinson's disease (PD), we predicted two clinical symptom scores, motor Unified Parkinson’s Disease Rating Scale (UPDRS) and total UPDRS \cite{tsanas2009accurate}. These data come from a six-month trial that aimed to assess the suitability of measurements of dysphonia (impairment of voice production) for telemonitoring of PD. The repeated measures data contains 5,875 voice recordings from the 42 individuals. The following variables were considered as covariates for predicting the two outcomes: subject identifier, age, sex, and sixteen biomedical voice measures. 

For each MTL estimator considered in the simulations, we examined predictive performance in the data analysis in terms of the CV MSE. We considered a clustered CV scheme with respect to the subject identifier (each subject's observations are always all together as training or test set observations across all CV folds) of k-fold/V-fold CV with ten folds, and so the MSE reported represents an honest, independent evaluation on a test dataset that was not seen during training. The results from the data analysis are presented in Table~\eqref{table::data}. 
%and Figure~\eqref{fig::data}.

\begin{table}[ht!]
\caption{Predictive performance, measured in terms of the cross-validated mean squared error for motor UPDRS, total UPDRS and overall outcome, for each multi-task estimator considered in the data application for predicting of Parkinson's disease symptom scores from biomedical voice measurements.}
\label{table::data}
\begin{center}
\begin{small}
\begin{sc}
\begin{tabular}{@{} @{} lrrrr @{} @{}}
\toprule
\textbf{Method} & \textbf{mUPDRS} & \textbf{tUPDRS} & \textbf{Overall}\\
\midrule
\centering
MT-HAL   & 58.9 & 101.0 & 79.9 \\
MT-lasso & 60.5 & 103.9 & 82.2 \\
MT-L21   & 60.9 & 104.2 & 82.6 \\
\bottomrule
\end{tabular}
\end{sc}
\end{small}
\end{center}
\vskip -0.30in
\end{table}

The MT-HAL resulted in the lowest cross-validated MSE out of all the MTL algorithms considered. However, we note that performance of all MTL algorithms in this setup is poor. This could be due to a low signal in collected covariates w.r.t. to the target outcomes. In a real-world application, we also typically advocate for the use of a CV-based ``super learner'' selector that, among a library of candidate MTL algorithms (including several variations of MT-HAL with different hyperparameter specifications), chooses the MTL algorithm with the lowest CV risk. 
%This setup has been proven to represent an asymptotically optimal system for estimation, including in MTL settings \cite{van2003unified,polley2010super}. 
A MTL ``super learner'' with a rich library of different MTL algorithms might be able to improve on the predictive performance, instead of considering each algorithm separately. 

% \begin{figure}[ht]
% \vskip 0.2in
% \begin{center}
% \centerline{\includegraphics[width=\columnwidth]{Figures/analysis_boxplots.pdf}}
% \caption{Boxplots of the squared error loss for each of the multi-task estimators considered in the data analysis for predicting two clinical symptom scores of Parkinson's disease from several biomedical voice measurements.}
% \label{fig::data}
% \end{center}
% \vskip -0.2in
% \end{figure}

\section{Discussion}\label{sec::discussion}

In this work, we propose a fully nonparametric approach for the multi-task learning problem. Our proposed framework simultaneously learns features, samples and task associations important for the common model, while imposing a shared sparse structure among similar tasks. The problem formulation imposes no assumptions on the relationship between predictors and outcomes, or task interconnections, other than a global bound on the variation norm and function space; these assumptions prove to be gather general in nature, often resulting in pathological examples if not respected. The proposed MTL algorithm attains a powerful dimension-free $o_p(n^{-1/4})$ (or better) convergence rate, and outperforms all considered sparsity-based MTL competitors across a wide range of DGPs: including nonlinear and linear setups, different levels of sparsity and task correlations, dimension of covariate space and sample sizes.

As part of future work, there are several directions to be explored. First, we can make the procedure more computationally scalable by relying on empirical loss minimization within nested Donsker classes \cite{schuler2022}. For online or sequential data, one can consider formulating the multi-task problem as only part of the flexible ensemble model \cite{malenica2023}. Along similar lines, the issue of ``when to share" can be alleviated by formulating the question as a loss-based problem as well. In particular, one could define a problem setup where cross-validation is used to pick among single- and multi-task models, even considering many different algorithms and task combinations.

% In the unusual situation where you want a paper to appear in the references without citing it in the main text, use \nocite
%\nocite{langley00}

\bibliography{multi_task_HAL}

\begin{thebibliography}{33}
\providecommand{\natexlab}[1]{#1}
\providecommand{\url}[1]{\texttt{#1}}
\expandafter\ifx\csname urlstyle\endcsname\relax
  \providecommand{\doi}[1]{doi: #1}\else
  \providecommand{\doi}{doi: \begingroup \urlstyle{rm}\Url}\fi

\bibitem[Benkeser \& {van der Laan}(2016)Benkeser and {van der
  Laan}]{benkeser2016}
Benkeser, D. and {van der Laan}, M.
\newblock {{T}he {H}ighly {A}daptive {L}asso {E}stimator}.
\newblock \emph{Proc Int Conf Data Sci Adv Anal}, 2016:\penalty0 689--696,
  2016.

\bibitem[Cipolla et~al.(2018)Cipolla, Gal, and Kendall]{kendall2018}
Cipolla, R., Gal, Y., and Kendall, A.
\newblock Multi-task learning using uncertainty to weigh losses for scene
  geometry and semantics.
\newblock In \emph{2018 IEEE/CVF Conference on Computer Vision and Pattern
  Recognition}, pp.\  7482--7491, 2018.
\newblock \doi{10.1109/CVPR.2018.00781}.

\bibitem[Crawshaw(2020)]{crawshaw2020}
Crawshaw, M.
\newblock Multi-task learning with deep neural networks: A survey, 2020.
\newblock URL \url{https://arxiv.org/abs/2009.09796}.

\bibitem[Dizaji et~al.(2021)Dizaji, Chen, and Huang]{dizaji2021}
Dizaji, K.~G., Chen, W., and Huang, H.
\newblock {{D}eep {L}arge-{S}cale {M}ultitask {L}earning {N}etwork for {G}ene
  {E}xpression {I}nference}.
\newblock \emph{J Comput Biol}, 28\penalty0 (5):\penalty0 485--500, May 2021.

\bibitem[Dudoit \& {van der Laan}(2005)Dudoit and {van der Laan}]{dudoit2003a}
Dudoit, S. and {van der Laan}, M.
\newblock Asymptotics of cross-validated risk estimation in estimator selection
  and performance assessment.
\newblock \emph{Statistical Methodology}, 2\penalty0 (2):\penalty0 131 -- 154,
  2005.
\newblock ISSN 1572-3127.
\newblock \doi{https://doi.org/10.1016/j.stamet.2005.02.003}.
\newblock URL
  \url{http://www.sciencedirect.com/science/article/pii/S1572312705000158}.

\bibitem[Gill et~al.(1995)Gill, {van der Laan}, and {van der
  Wellner}]{gill1995}
Gill, R., {van der Laan}, M., and {van der Wellner}, J.
\newblock Inefficient estimators of the bivariate survival function for three
  models.
\newblock \emph{Annales de l'I.H.P. Probabilités et statistiques}, 31\penalty0
  (3):\penalty0 545--597, 1995.
\newblock URL \url{http://eudml.org/doc/77521}.

\bibitem[Jawanpuria \& Nath(2012)Jawanpuria and Nath]{jawanpuria2012}
Jawanpuria, P. and Nath, J.~S.
\newblock A convex feature learning formulation for latent task structure
  discovery, 2012.
\newblock URL \url{https://arxiv.org/abs/1206.4611}.

\bibitem[Liu et~al.(2019)Liu, He, Chen, and Gao]{liu2019}
Liu, X., He, P., Chen, W., and Gao, J.
\newblock Multi-task deep neural networks for natural language understanding.
\newblock In \emph{Proceedings of the 57th Annual Meeting of the Association
  for Computational Linguistics}, pp.\  4487--4496, Florence, Italy, July 2019.
  Association for Computational Linguistics.
\newblock \doi{10.18653/v1/P19-1441}.
\newblock URL \url{https://aclanthology.org/P19-1441}.

\bibitem[Lounici et~al.(2009)Lounici, Pontil, Tsybakov, and van~de
  Geer]{lounici2009}
Lounici, K., Pontil, M., Tsybakov, A.~B., and van~de Geer, S.
\newblock Taking advantage of sparsity in multi-task learning, 2009.
\newblock URL \url{https://arxiv.org/abs/0903.1468}.

\bibitem[Malenica et~al.(2023)Malenica, Phillips, Chambaz, Hubbard, Pirracchio,
  and van~der Laan]{malenica2023}
Malenica, I., Phillips, R.~V., Chambaz, A., Hubbard, A.~E., Pirracchio, R., and
  van~der Laan, M.~J.
\newblock Personalized online ensemble machine learning with applications for
  dynamic data streams.
\newblock \emph{Statistics in Medicine}, pp.\  1-- 32, 2023.
\newblock \doi{https://doi.org/10.1002/sim.9655}.
\newblock URL \url{https://onlinelibrary.wiley.com/doi/abs/10.1002/sim.9655}.

\bibitem[Misra et~al.(2016)Misra, Shrivastava, Gupta, and Hebert]{misra2016}
Misra, I., Shrivastava, A., Gupta, A., and Hebert, M.
\newblock Cross-stitch networks for multi-task learning.
\newblock In \emph{2016 IEEE Conference on Computer Vision and Pattern
  Recognition (CVPR)}, pp.\  3994--4003, Los Alamitos, CA, USA, jun 2016. IEEE
  Computer Society.
\newblock \doi{10.1109/CVPR.2016.433}.
\newblock URL \url{https://doi.ieeecomputersociety.org/10.1109/CVPR.2016.433}.

\bibitem[Obozinski et~al.(2011)Obozinski, Wainwright, and
  Jordan]{obozinski2011}
Obozinski, G., Wainwright, M., and Jordan, M.
\newblock {Support union recovery in high-dimensional multivariate regression}.
\newblock \emph{The Annals of Statistics}, 39\penalty0 (1):\penalty0 1 -- 47,
  2011.
\newblock \doi{10.1214/09-AOS776}.
\newblock URL \url{https://doi.org/10.1214/09-AOS776}.

\bibitem[Ruder(2017)]{ruder2017}
Ruder, S.
\newblock An overview of multi-task learning in deep neural networks, 2017.
\newblock URL \url{https://arxiv.org/abs/1706.05098}.

\bibitem[Schuler \& {van der Laan}(2022)Schuler and {van der
  Laan}]{schuler2022}
Schuler, A. and {van der Laan}, M.
\newblock The selectively adaptive lasso, 2022.
\newblock URL \url{https://arxiv.org/abs/2205.10697}.

\bibitem[Simon \& Tibshirani(2012)Simon and Tibshirani]{simon2012}
Simon, N. and Tibshirani, R.
\newblock {Standardization and the group lasso penalty}.
\newblock \emph{Stat Sin}, 22\penalty0 (3):\penalty0 983--1001, Jul 2012.

\bibitem[Sodhani et~al.(2021)Sodhani, Zhang, and Pineau]{sodhani2021}
Sodhani, S., Zhang, A., and Pineau, J.
\newblock Multi-task reinforcement learning with context-based representations.
\newblock In Meila, M. and Zhang, T. (eds.), \emph{Proceedings of the 38th
  International Conference on Machine Learning}, volume 139 of
  \emph{Proceedings of Machine Learning Research}, pp.\  9767--9779. PMLR,
  18--24 Jul 2021.
\newblock URL \url{https://proceedings.mlr.press/v139/sodhani21a.html}.

\bibitem[Sun et~al.(2019)Sun, Shao, Li, Liu, Yan, Qiu, and Huang]{sun2019}
Sun, T., Shao, Y., Li, X., Liu, P., Yan, H., Qiu, X., and Huang, X.
\newblock Learning sparse sharing architectures for multiple tasks, 2019.
\newblock URL \url{https://arxiv.org/abs/1911.05034}.

\bibitem[Tang et~al.(2020)Tang, Pino, Wang, Ma, and Genzel]{tang2020}
Tang, Y., Pino, J., Wang, C., Ma, X., and Genzel, D.
\newblock A general multi-task learning framework to leverage text data for
  speech to text tasks, 2020.
\newblock URL \url{https://arxiv.org/abs/2010.11338}.

\bibitem[Tsanas et~al.(2009)Tsanas, Little, McSharry, and
  Ramig]{tsanas2009accurate}
Tsanas, A., Little, M., McSharry, P., and Ramig, L.
\newblock Accurate telemonitoring of parkinson’s disease progression by
  non-invasive speech tests.
\newblock \emph{Nature Precedings}, pp.\  1--1, 2009.

\bibitem[{Tsiatis}(2006)]{tsiatis2006}
{Tsiatis}, A.
\newblock \emph{Semiparametric Theory and Missing Data}.
\newblock Springer New York, NY, 2006.
\newblock ISBN 978-0-387-37345-4.

\bibitem[van~der Laan(2017)]{vanderlaan2017hal}
van~der Laan, M.
\newblock {{A} {G}enerally {E}fficient {T}argeted {M}inimum {L}oss {B}ased
  {E}stimator based on the {H}ighly {A}daptive {L}asso}.
\newblock \emph{Int J Biostat}, 13\penalty0 (2), 10 2017.

\bibitem[{van der Laan} \& Bibaut(2017){van der Laan} and
  Bibaut]{laan2017uniform_consist_hal}
{van der Laan}, M. and Bibaut, A.
\newblock Uniform consistency of the highly adaptive lasso estimator of
  infinite dimensional parameters, 2017.
\newblock URL \url{https://arxiv.org/abs/1709.06256}.

\bibitem[{van der Laan} \& Dudoit(2003){van der Laan} and Dudoit]{dudoit2003b}
{van der Laan}, M. and Dudoit, S.
\newblock Unified cross-validation methodology for selection among estimators
  and a general cross-validated adaptive epsilon-net estimator: Finite sample
  oracle inequalities and examples, 2003.
\newblock URL \url{https://biostats.bepress.com/ucbbiostat/paper130}.

\bibitem[{van der Laan} et~al.(2006){van der Laan}, Dudoit, and {van der
  Vaart}]{laan2006oracle}
{van der Laan}, M., Dudoit, S., and {van der Vaart}, A.
\newblock {The cross-validated adaptive epsilon-net estimator}.
\newblock \emph{Statistics \& Risk Modeling}, 24\penalty0 (3):\penalty0 1--23,
  December 2006.
\newblock URL \url{https://ideas.repec.org/a/bpj/strimo/v24y2006i3p23n4.html}.

\bibitem[van~der Vaart \& Wellner(2013)van~der Vaart and
  Wellner]{vanderVaartWellner96}
van~der Vaart, A. and Wellner, J.
\newblock \emph{Weak Convergence and Empirical Processes}.
\newblock Springer-Verlag New York, 03 2013.
\newblock ISBN 9781475725452.

\bibitem[{van der Vaart} et~al.(2006){van der Vaart}, Dudoit, and {van der
  Laan}]{vaart2006}
{van der Vaart}, A., Dudoit, S., and {van der Laan}, M.
\newblock {Oracle inequalities for multi-fold cross validation}.
\newblock \emph{Statistics \& Risk Modeling}, 24\penalty0 (3):\penalty0 1--21,
  December 2006.
\newblock URL \url{https://ideas.repec.org/a/bpj/strimo/v24y2006i3p21n3.html}.

\bibitem[Vandenhende(2022)]{vandenhende2022}
Vandenhende, S.
\newblock Multi-task learning for visual scene understanding, 2022.
\newblock URL \url{https://arxiv.org/abs/2203.14896}.

\bibitem[Wang \& Sun(2022)Wang and Sun]{wang2022}
Wang, J. and Sun, L.
\newblock Multi-task personalized learning with sparse network lasso.
\newblock In Raedt, L.~D. (ed.), \emph{Proceedings of the Thirty-First
  International Joint Conference on Artificial Intelligence, {IJCAI-22}}, pp.\
  3516--3522. International Joint Conferences on Artificial Intelligence
  Organization, 7 2022.
\newblock \doi{10.24963/ijcai.2022/488}.
\newblock URL \url{https://doi.org/10.24963/ijcai.2022/488}.
\newblock Main Track.

\bibitem[Yuan \& Lin(2006)Yuan and Lin]{yuan2006}
Yuan, M. and Lin, Y.
\newblock Model selection and estimation in regression with grouped variables.
\newblock \emph{Journal of the Royal Statistical Society: Series B (Statistical
  Methodology)}, 68\penalty0 (1):\penalty0 49--67, 2006.
\newblock \doi{https://doi.org/10.1111/j.1467-9868.2005.00532.x}.
\newblock URL
  \url{https://rss.onlinelibrary.wiley.com/doi/abs/10.1111/j.1467-9868.2005.00532.x}.

\bibitem[Zhang \& Yang(2022)Zhang and Yang]{zhang2022}
Zhang, Y. and Yang, Q.
\newblock A survey on multi-task learning.
\newblock \emph{IEEE Transactions on Knowledge and Data Engineering},
  34\penalty0 (12):\penalty0 5586--5609, 2022.
\newblock \doi{10.1109/TKDE.2021.3070203}.

\bibitem[Zhang \& Yeung(2014)Zhang and Yeung]{zhang2014}
Zhang, Y. and Yeung, D.-Y.
\newblock A regularization approach to learning task relationships in multitask
  learning.
\newblock \emph{ACM Trans. Knowl. Discov. Data}, 8\penalty0 (3), jun 2014.
\newblock ISSN 1556-4681.
\newblock \doi{10.1145/2538028}.
\newblock URL \url{https://doi.org/10.1145/2538028}.

\bibitem[Zhang et~al.(2010)Zhang, Yeung, and Xu]{zhang2010}
Zhang, Y., Yeung, D.-Y., and Xu, Q.
\newblock Probabilistic multi-task feature selection.
\newblock In Lafferty, J., Williams, C., Shawe-Taylor, J., Zemel, R., and
  Culotta, A. (eds.), \emph{Advances in Neural Information Processing Systems},
  volume~23. Curran Associates, Inc., 2010.
\newblock URL
  \url{https://proceedings.neurips.cc/paper/2010/file/839ab46820b524afda05122893c2fe8e-Paper.pdf}.

\bibitem[Zhuang et~al.(2019)Zhuang, Qi, Duan, Xi, Zhu, Zhu, Xiong, and
  He]{zhuang2019}
Zhuang, F., Qi, Z., Duan, K., Xi, D., Zhu, Y., Zhu, H., Xiong, H., and He, Q.
\newblock A comprehensive survey on transfer learning, 2019.
\newblock URL \url{https://arxiv.org/abs/1911.02685}.

\end{thebibliography}
\bibliographystyle{icml2023}

%%%%%%%%%%%%%%%%%%%%%%%%%%%%%%%%%%%%%%%%%%%%%%%%%%%%%%%%%%%%%%%%%%%%%%%%%%%%%%%
%%%%%%%%%%%%%%%%%%%%%%%%%%%%%%%%%%%%%%%%%%%%%%%%%%%%%%%%%%%%%%%%%%%%%%%%%%%%%%%
% APPENDIX
%%%%%%%%%%%%%%%%%%%%%%%%%%%%%%%%%%%%%%%%%%%%%%%%%%%%%%%%%%%%%%%%%%%%%%%%%%%%%%%
%%%%%%%%%%%%%%%%%%%%%%%%%%%%%%%%%%%%%%%%%%%%%%%%%%%%%%%%%%%%%%%%%%%%%%%%%%%%%%%
\newpage
\appendix
\onecolumn

\section{Mixed-norm penalty controls the sectional variation norm}

In the following we give a sketch argument that the mixed-norm penalty $l_{2,1}$ still controls the sectional variation norm of a cadlag function. First, let $\psi \in \mathcal{H}$ and $\|\psi\|_{var} < \infty$. Then by \cite{gill1995} we can write $\psi(w)$ as
\begin{equation*}
    \psi(w) = \psi(0) + \sum_{s \subset \{1, \ldots, d\}} \int_{O_s}^{\tau_s} \mathbbm{1}(u \leq w_s)\psi_s(du).
\end{equation*}
We can approximate $\psi$ using a discrete measure $\psi_m$ and $m$ support points. In particular, let $s$ denote a section in $\{1,\ldots,d\}$ with $t$ a knot point, such that $(u_{s,t}:t)$ denotes all the support. Correspondingly, let $d\psi_{m,s,t}$ denote the pointmass assigned to $u_{s,t}$ by $\psi_m$, resulting in an approximation of $\psi_m(w)$ written as 
\begin{equation}\label{eq::approximation}
    \psi(0) + \sum_{s \subset \{1, \ldots, d\}} \sum_t \mathbbm{1}(u_{s,t} \leq w_s) d\psi_{m,s,t}.
\end{equation}
Note that $\mathbbm{1}(u_{s,t} \leq w_s)$ is a basis function, with the corresponding coefficient $d\psi_{m,s,t}$. The discrete approximation of $\psi(w)$ is a linear combination of basis functions summed over all sections $s$ and knots $t$. The sum of absolute values of $d\psi_{m,s,t}$ then corresponds to the variation norm of $\psi$, 
\begin{equation*}
    \|\psi_m\|_{var} = \psi(0) + \sum_{s \subset \{1,\ldots,d\}} \sum_t |d\psi_{m,s,t}|.
\end{equation*}

As defined in Section \ref{sec::algorithm}, let $\Tilde{w}_{s,i}$ denote an observed value $\Tilde{w}_{s,i} = \{\Tilde{w}_{c,i} : c \in s\}$ for subset $s$ with $i = 1, \ldots, n_k$ in task $k$. Applying result in Equation \eqref{eq::approximation} for the support defined by the $n_k$ samples for each task $k$, we derive to the following new approximation for $\psi(w)$:
\begin{equation*}
    \psi(0) + \sum_{s \subset \{1, \ldots, d\}} \sum_{k=1}^K \sum_{i=1}^{n_k} \mathbbm{1}(\Tilde{w}_{s,i} \leq w_s) d\psi_{m,s,t}.
\end{equation*}
where we can define $\phi_{s,i} = \mathbbm{1}(\Tilde{w}_{s,i} \leq w_s)$ and $d\psi_{m,s,t} = \beta_{s,i}$ to ease notation. The variation norm of $\psi$ is then 
\begin{equation*}
    \|\psi_m\|_{var} = \psi(0) + \sum_{s \subset \{1,\ldots,d\}} \sum_{k=1}^K \sum_{i=1}^{n_k} |\beta_{s,i}|, 
\end{equation*}
which corresponds to the $l_1$ norm for the following optimization problem
\begin{equation*}
   \psi_n = \argmin{\psi \in \mathcal{H}_M} P_n L(\psi)
\end{equation*}
with 
\[
\mathcal{H}_{M} = \begin{cases}
    \psi \in \mathcal{H} \\
    \text{s.t.} \ \|\psi\|_{var} \leq M.
    \end{cases}
\]
If $\psi \in \mathcal{H}$, we need the true variation norm of $\psi$ to be smaller than some universal constant M, which is equivalent to the amount of "allowance" given by the $l_1$ penalty (allowed upper bound on the sum of absolute values of the coefficients). It's straight forward to show that
\begin{align*}
\|\beta\|_{1} &= \sum_{p=1}^N \|\beta_p\|_{1}
%= \sum_{p=1}^N \sum_{k=1}^K |\beta_{p}^k| 
= \sum_{p=1}^N \left(\sum_{k=1}^K \sqrt{|\beta_{p}^k|^2}\right)  \\
&\geq \sum_{p=1}^N \left(\sqrt{\sum_{k=1}^K |\beta_{p}^k|^2}\right) = \sum_{p=1}^N \|\beta_p\|_{2}  \\
&= \|\beta\|_{2,1},
\end{align*}
proving that $l_{2,1}$ norm is less than or equal to the variation norm, thus controlling the amount of allowed variation. 

\section{Additional Simulations}

\subsection{Different sample sizes}

In the following, we provide results of additional simulations at various sample sizes. In particular, we report performance of MT-HAL, MT-lasso and MT-L21 for $n=(50,100,200, \ldots, 800,900,1000)$. All additional simulations correspond to nonlinear data-generating processes (DGPs) as described in Section \ref{sec::simulations}. We consider nonlinear DGPs across various $n$ with (1) high level of sparsity ($60\%$, ``H") vs. low sparsity ($20\%$, ``L" ); (2) tasks with the same level of sparsity (``S") vs. distinct sparsity levels across tasks (``D"). For all simulations, we consider $K=5$ tasks and different number of samples for each $k$, in order to encourage different level of importance for each task in the optimization procedure. The number of covariates remains the same across all tasks ($d=6$). We report the mean squared error (MSE) at each sample size and DGP in Figure \ref{fig::mse}, as well as coefficient precision and accuracy in Figures \ref{fig::precision} and \ref{fig::recall}. All reported results are calculated on a separate test data over $100$ Monte Carlo simulations.

As reported for $n=600$ in the main simulation results, MT-HAL achieves the lowest MSE across all DGPs and sample sizes considered. The improvement seen at $n=600$ persists even for small sample sizes ($n=50$), and consistently shows reduction in the MSE by more than a half. In terms of coefficient precision, the hardest settings across all algorithms and sample sizes are, as expected, DGPs with high sparsity, where MT-lasso seems to perform the best. However, accuracy results also demonstrate that MT-lasso tends to produce a lot of false negatives. Except for the very small samples sizes, MT-HAL results in best accuracy and comparable precision across all considered DGP settings. 

\subsection{Linear Setup with Interactions}

The data-generating processes corresponding to simulation ``LHS'', ``LLS'', ``LHD'' and ``LLD'' are linear models with $d=6$ covariates and normally distributed coefficients. The setup for linear simulations is the same as for nonlinear, we just omit the nonlinear transformations of the covariate space. For example, true coefficients for simulation ``LHS'' are $\beta_{0,LHS} = (\beta_{0,1}, \ldots, \beta_{0,|LHS|}, 0, \ldots, 0)$. Each $\beta_{0,j}$ for $1 \leq j \leq |LHS|$ is sampled from a standard normal distribution, and the error term is normal with final coefficient of $0.3$. For simulations with the same level of sparsity across tasks, highest level of sparsity was $60\%$, whereas the lowest was $20\%$, as for the nonlinear DGPs. For different sparsity profiles across tasks, we considered random deviations from true level of sparsity across different tasks ($40\%-80\%$ for high level of sparsity and $0.05\%-40\%$ for low). Covariates with nonzero coefficients consist of binary, categorical and continuous variables. Final outcome regression consist of only linear terms, but we add few interactions as well (up to second order). High-sparsity settings put most nonzero coefficients on terms that are not interactions, while low-sparsity setup includes interactions. We report results for linear simulations at $n=600$ in Table \eqref{table::table_sim2}. 

As expected, there is a vast improvement in overall performance for all methods for an easier (linear) setup. For high-sparsity simulations (LHS and LHD), the true model is linear with no interactions, hence both MT-lasso and MT-L21 operate within the true model. As MT-HAL starts from a nonparametric space, it is not surprising to see MT-lasso and MT-L21 perform better in this setup; it is, however, usually unrealistic to know the true DGP in advance. It is worth nothing that, even in a completely linear setting, as soon as the true DGP contains few interactions, performance of MT-lasso and MT-L21 significantly drops in terms of the MSE. As observed in nonlinear simulations, MT-HAL preserves its advantage in terms of MSE performance for high-sparsity simulations; when MT-lasso and MT-l21 operate within their true model, MT-HAL remains very competitive.

\subsection{High-dimensional Setup}

Finally, we demonstrate performance of the multi-task HAL in high-dimensions in Table \ref{table::table_sim1_hd}. In particular, the data-generating processes corresponding to simulation ``HNHS'', ``HNLS'', ``HNHD'' and ``HNLD'' are high-dimensional nonlinear models with $d=20$ covariates per each task and normally distributed coefficients. The DGPs in the high-dimensional setting are as previously described in Section \ref{sec::simulations}. The true coefficients for simulation ``HNHS'' are $\beta_{0,HNHS} = (\beta_{0,1}, \ldots, \beta_{0,|NHS|}, 0, \ldots, 0)$, where each $\beta_{0,j}$ for $1 \leq j \leq |HNHS|$ is sampled from a standard normal distribution. The corresponding error term is sampled from $\mathcal{N}(0,0.1)$, with final coefficient of $0.3$. Highest and lowest level of sparsity was $60\%$ and $20\%$ for simulations where a lot of sharing across tasks is expected. For different sparsity profiles across tasks, we considered random deviations from true level of sparsity across different tasks ($40\%-80\%$ for high level of sparsity and $0.05\%-40\%$ for low). Covariates with nonzero coefficients are transformed via exponential, logarithmic, cosine and squared operations of predictors and predictor interactions. Final MSE, coefficient precision and accuracy results for high-dimensional nonlinear simulations at $n=600$ are shown in Table \eqref{table::table_sim1_hd}.

Similarly to other nonlinear simulations with smaller number of covariates (and less continuous predictors),  
MT-HAL remains the MTL algorithm with the smallest mean squared error. Across all DGPs considered, it consistently shows reduction in MSE by more than a half (as seen in previous simulations as well). Precision and accuracy show a significant decrease for all considered MTL algorithms, as expected for high-dimensional highly nonlinear settings. Despite not getting all non-zero coefficients right, MSE remains surprisingly low for all considered methods. As observed previously, MT-HAL universally produces better accuracy and comparable precision in terms of the non-zero coefficients, compared to considered competitors. 

\begin{table}[ht!]
\caption{Mean squared error (MSE), precision (``Prec") and accurarcy (``Accu'') for each of the linear simulation setups with $d=6$ covariates, $K=5$ tasks, and total of $n=600$ samples split between the $K$ tasks as $\{100,100,150,150,100\}$. Reported results were generated across $100$ Monte Carlo simulations.}\label{table::table_sim2}
\vskip 0.15in
\begin{center}
\begin{small}
\begin{sc}
\begin{tabular}{@{} @{} lrrrrr @{} @{}}
\toprule
\textbf{Setup} & \textbf{Method} & \textbf{MSE} & \textbf{Prec $\%$}  & \textbf{Accu $\%$} \\
\midrule
\centering
LHS & MT-HAL   & 0.062 & 95.2 & 85.7  \\
LHS & MT-lasso & 0.054 & 100 & 76.4  \\
LHS & MT-L21   & 0.014 & 100 & 90.3  \\
\midrule
LLS & MT-HAL   & 0.144 & 94.1 & 74.1  \\
LLS & MT-lasso & 0.395 & 99.8 & 65.1  \\
LLS & MT-L21   & 0.324 & 99.7 & 74.9  \\
\midrule
LHD & MT-HAL   & 0.060 & 94.9 & 85.2  \\
LHD & MT-lasso & 0.035 & 100 & 78.7  \\
LHD & MT-L21   & 0.018 & 100 & 85.7  \\
\midrule
LLD & MT-HAL   & 0.284 & 95.9 & 78.0  \\
LLD & MT-lasso & 0.717 & 100 & 64.5  \\
LLD & MT-L21   & 0.644 & 99.8 & 77.2  \\
\bottomrule
\end{tabular}
\end{sc}
\end{small}
\end{center}
\vskip -0.1in
\end{table}

\begin{table}[ht!]
\caption{Mean squared error (MSE), precision (``Prec") and accurarcy (``Accu") for each of the high-dimensional nonlinear simulation setups with $d=20$ covariates, $K=5$ tasks, and total of $n=600$ samples split between the $K$ tasks as $\{100,100,150,150,100\}$. Reported results were generated across $100$ Monte Carlo simulations.}\label{table::table_sim1_hd}
\vskip 0.15in
\begin{center}
\begin{small}
\begin{sc}
\begin{tabular}{@{} @{} lrrrrr @{} @{}}
\toprule
\textbf{Setup} & \textbf{Method} & \textbf{MSE} & \textbf{Prec $\%$}  & \textbf{Accu $\%$} \\
\midrule
\centering
HNHS & MT-HAL   & 0.446 & 20.8 & 37.8  \\
HNHS & MT-lasso & 0.863 & 44.2 & 27.5  \\
HNHS & MT-L21   & 0.808 & 39.0 & 35.7  \\
\midrule
HNLS & MT-HAL   & 0.431 & 38.2 & 50.1  \\
HNLS & MT-lasso & 1.01 & 46.1 & 31.3  \\
HNLS & MT-L21   & 0.812 & 45.4 & 40.0  \\
\midrule
HNHD & MT-HAL   & 0.425 & 37.3 & 36.7  \\
HNHD & MT-lasso & 0.841 & 34.3 & 26.4  \\
HNHD & MT-L21   & 0.809 & 42.6 & 33.7  \\
\midrule
HNLD & MT-HAL   & 0.451 & 36.5 & 31.7  \\
HNLD & MT-lasso & 1.01 & 46.1 & 30.5  \\
HNLD & MT-L21   & 0.842 & 42.9 & 39.1  \\
\bottomrule
\end{tabular}
\end{sc}
\end{small}
\end{center}
\vskip -0.1in
\end{table}

%%%%%%%%%%%%%%%%%%%%%%%%%%%%%%%%%%%%%%%%%%%%%%%%%%%%%%%%%%%%%%%%%%%%%%%%%%%%%%%
%%%%%%%%%%%%%%%%%%%%%%%%%%%%%%%%%%%%%%%%%%%%%%%%%%%%%%%%%%%%%%%%%%%%%%%%%%%%%%%

%Instructions for Figures: 
%Lines should be dark and at least 0.5~points thick for purposes of reproduction, and text should not appear on a gray background.

%Always place two-column figures at the top or bottom of the page.

\begin{figure}[ht]
\vskip 0.2in
\begin{center}
\centerline{\includegraphics[width=\columnwidth]{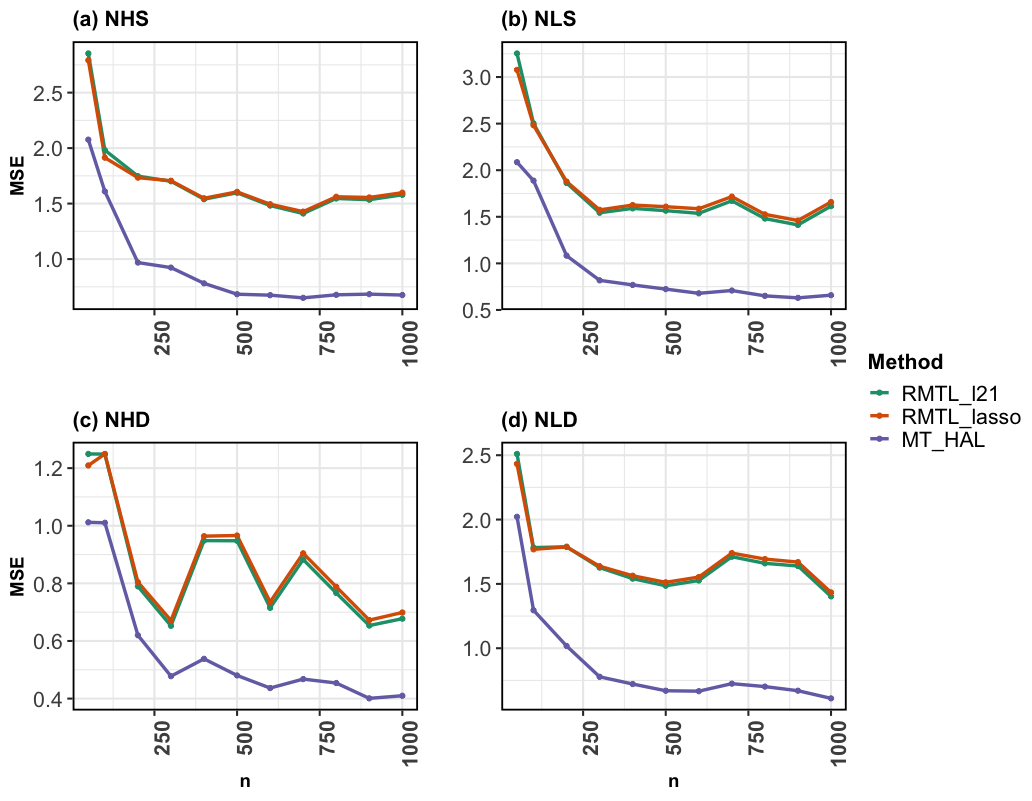}}
\caption{Mean Squared Error (MSE) at sample sizes $n=(50,100,200, \ldots, 800,900,1000)$ for each of the nonlinear simulation setups: nonlinear, high sparsity, same sparsity profile across tasks (``NHS"); nonlinear, low sparsity, same sparsity profile across tasks (``NLS"); nonlinear, high sparsity, different sparsity profile across tasks (``NHD"); nonlinear, low sparsity, different sparsity profile across tasks (`NLD"). All simulation setups contain $d=6$ covariates and $K=5$ tasks. Reported results were generated across $100$ Monte Carlo simulations.}
\label{fig::mse}
\end{center}
\vskip -0.2in
\end{figure}

\begin{figure}[ht]
\vskip 0.2in
\begin{center}
\centerline{\includegraphics[width=\columnwidth]{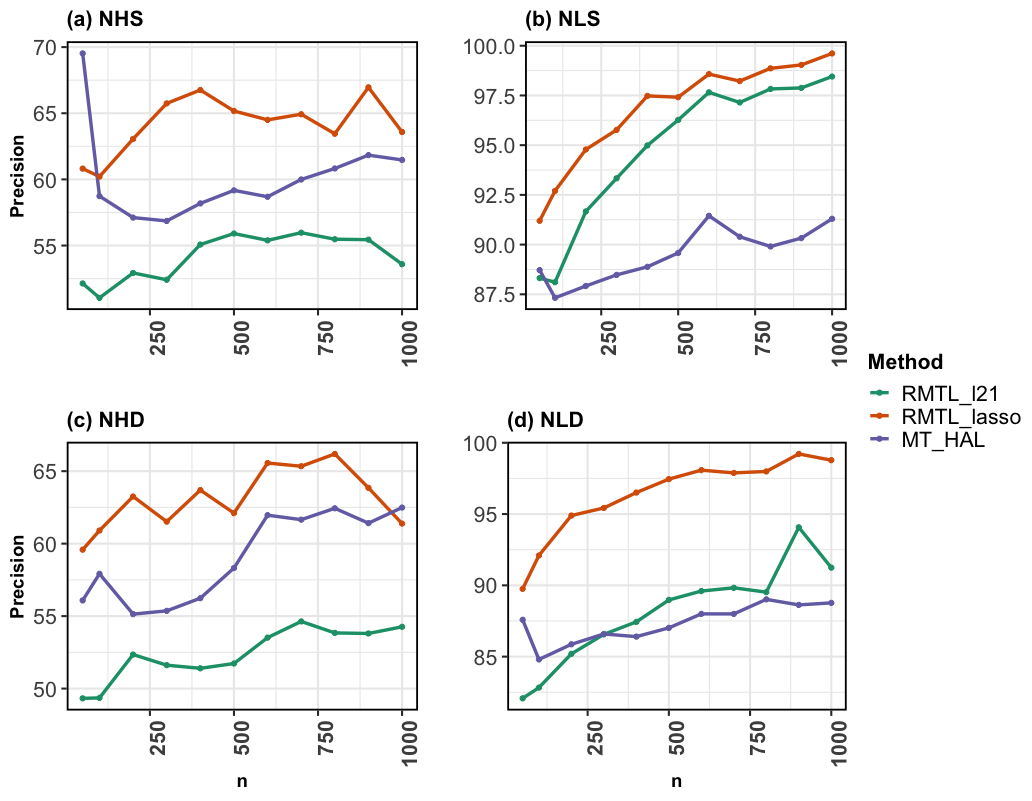}}
\caption{Coefficient precision at sample sizes $n=(50,100,200, \ldots, 800,900,1000)$ for each of the nonlinear simulation setups: nonlinear, high sparsity, same sparsity profile across tasks (``NHS"); nonlinear, low sparsity, same sparsity profile across tasks (``NLS"); nonlinear, high sparsity, different sparsity profile across tasks (``NHD"); nonlinear, low sparsity, different sparsity profile across tasks (`NLD"). All simulation setups contain $d=6$ covariates and $K=5$ tasks. Reported results were generated across $100$ Monte Carlo simulations.}
\label{fig::precision}
\end{center}
\vskip -0.2in
\end{figure}

\begin{figure}[ht]
\vskip 0.2in
\begin{center}
\centerline{\includegraphics[width=\columnwidth]{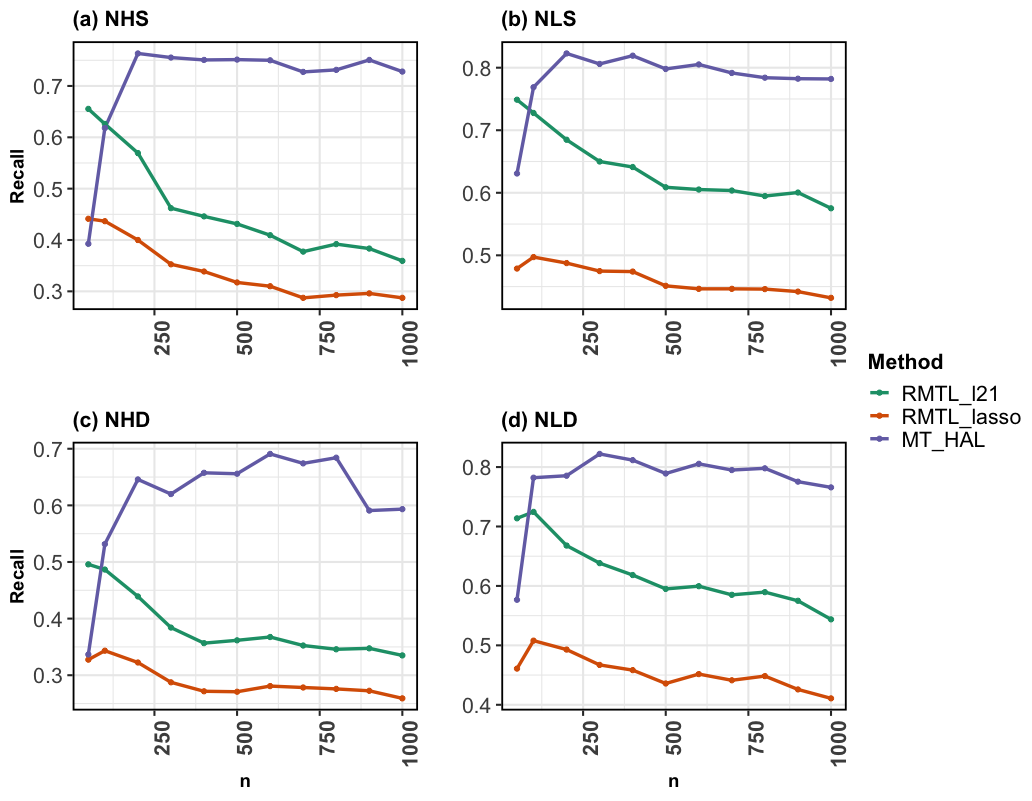}}
\caption{Coefficient accuracy at sample sizes $n=(50,100,200, \ldots, 800,900,1000)$ for each of the nonlinear simulation setups: nonlinear, high sparsity, same sparsity profile across tasks (``NHS"); nonlinear, low sparsity, same sparsity profile across tasks (``NLS"); nonlinear, high sparsity, different sparsity profile across tasks (``NHD"); nonlinear, low sparsity, different sparsity profile across tasks (`NLD"). All simulation setups contain $d=6$ covariates and $K=5$ tasks. Reported results were generated across $100$ Monte Carlo simulations.}
\label{fig::recall}
\end{center}
\vskip -0.2in
\end{figure}

\end{document}